\documentclass{article}
\usepackage{spconf}

\def\thisismainpaper{0}   

\usepackage[dvips]{graphicx}
\usepackage[cmex10]{amsmath}
\usepackage{multirow}
\usepackage{epsfig}
\usepackage{mathrsfs}
\usepackage{amssymb}
\usepackage{comment}
\usepackage{hyperref}
\usepackage{dsfont}
\usepackage{array}
\usepackage{amsthm}
\usepackage{booktabs}
\usepackage{blkarray}
\usepackage{enumerate}
\usepackage{url}
\usepackage{chemarrow}
\usepackage{float}
\usepackage{color}
\usepackage[lined,ruled,linesnumbered,noend]{algorithm2e}
\SetKwRepeat{Do}{do}{while}%
\usepackage{epsf,psfrag}
\usepackage{epsfig}
\usepackage{bbm}
\usepackage{bm}

\if\thisismainpaper1
\fi

\theoremstyle{definition}
\newtheorem{theorem}{Theorem}[section]
\newtheorem{lemma}[theorem]{Lemma}

\def\scale#1{{\small #1}}

\newcommand{\E}{{\rm I\kern-.3em E}}

\def\diag{\operatorname{diag}}

\title{Energy-efficient Decentralized Learning via Graph Sparsification}
\name{Xusheng Zhang, Cho-Chun Chiu, and Ting He
}
\address{Pennsylvania State University, University Park, PA 16802, USA}
\begin{document}
%
\maketitle
\begin{abstract}

This work aims at improving the energy efficiency of decentralized learning by optimizing the mixing matrix, which controls the communication demands during the learning process. Through rigorous analysis based on a state-of-the-art decentralized learning algorithm, the problem is formulated as a bi-level optimization, with the lower level solved by graph sparsification. A solution with guaranteed performance is proposed for the special case of fully-connected base topology and a greedy heuristic is proposed for the general case. Simulations based on real topology and dataset show that the proposed solution can lower the energy consumption at the busiest node by $54\%$--$76\%$ while maintaining the quality of the trained model. 
\end{abstract}
%
%
\section{Introduction}
\label{sec:intro}

Learning from decentralized data \cite{McMahan17AISTATS} is an emerging machine learning paradigm that has found many applications~\cite{Kairouz21book}.

Communication efficiency has been a major consideration in designing learning algorithms, as the cost in communicating model updates, e.g., communication time, bandwidth consumption, and energy consumption, dominates the total operation cost in many application scenarios~\cite{McMahan17AISTATS}. Existing works on reducing this cost can be broadly classified into (i) model compression for reducing the cost per communication~\cite{Compression1, Lan23} 
and (ii) hyperparameter optimization for reducing the number of communications until convergence~\cite{
Wang19JSAC}. The two approaches are orthogonal and can be applied jointly. 

In this work, we focus on hyperparameter optimization in the decentralized learning setting, where nodes communicate with neighbors according to a given base topology~\cite{Lian17NIPS}. 
To this end, we adopt a recently proposed optimization framework from \cite{Chiu23JSAC} that allows for systematic design of a critical hyperparameter in decentralized learning, the \emph{mixing matrix}, to minimize a generally-defined cost measure. The choice of mixing matrix as the design parameter utilizes the observation from \cite{MATCHA19} that \emph{not all the links are equally important for convergence}. Hence, instead of communicating over all the links at the same frequency as in most of the existing works~\cite{McMahan17AISTATS,
Wang19JSAC}, communicating on different links with different frequencies can further improve the communication efficiency. However, the existing mixing matrix designs \cite{MATCHA19,Chiu23JSAC} fall short at addressing a critical cost measure in wireless networks: \emph{energy consumption at the busiest node}. Although energy consumption is considered in \cite{Chiu23JSAC}, its cost model only captures the total energy consumption over all the nodes. 
In this work, we address this gap based on a rigorous theoretical foundation. \looseness=-1

\subsection{Related Work}\label{subsec:Related Work}

\noindent \textbf{Decentralized learning algorithms.} The standard algorithm for learning under a fully decentralized architecture was an algorithm called Decentralized Parallel Stochastic Gradient Descent (D-PSGD)~\cite{Lian17NIPS}, which was shown to achieve the same computational complexity but a lower communication complexity than training via a central server. 
Since then a number of improvements have been developed, e.g., \cite{
ICMLhonor},  
but these works only focused on the number of iterations. \looseness=-1

\noindent\textbf{Communication cost reduction.} One line of works tried to reduce the amount of data per communication through model compression, e.g., \cite{Compression1, Compression2,Compression3, Zhang20INFOCOM}.
Another line of works reduced the frequency of communications, e.g., \cite{sysml19,Ngu19INFOCOM,
Wang19JSAC}. 
by designing an optimized frequency~\cite{sysml19,Ngu19INFOCOM} or an adaptive frequency~\cite{Wang19JSAC}. A unified analysis of the cost-convergence tradeoff of such solutions was provided in \cite{Wang21JMLR}. 
Later works \cite{Singh20CDC,Singh21JSAIT} started to combine model compression and infrequent communications. 
Recently, it was recognized that better tradeoffs can be achieved by activating subsets of links, e.g., via event-based triggers \cite{Singh20CDC,Singh21JSAIT} or predetermined mixing matrices \cite{MATCHA19, Chiu23JSAC}. 
Our work is closest to \cite{MATCHA19,Chiu23JSAC} by also designing the mixing matrix, but we address a different objective of maximum per-node energy consumption.  \looseness=-1

\noindent\textbf{Mixing matrix design.} 
Mixing matrix design has been considered in the classical problem of distributed averaging, e.g., \cite{Boyd04,BoydGossip} 
designed a mixing matrix with the fastest convergence to $\epsilon$-average and \cite{seqW1} 
designed a sequence of mixing matrices to achieve exact average in finite time. 
\if\thisismainpaper1
In contrast, fewer works have addressed the design of mixing matrices in decentralized learning, e.g., \cite{
MATCHA19,Chiu23JSAC}. 
We refer to \cite{Xusheng23:report} for a more complete overview of related works. 
\else
In contrast, fewer works have addressed the design of mixing matrices in decentralized learning~\cite{Neglia19INFOCOM,
MATCHA19,Chiu23JSAC, Marfoq20throughput, lebar23}. Out of these, most focused on optimizing the training time, either by minimizing the time per iteration on computation \cite{Neglia19INFOCOM} or communication \cite{MATCHA19,Marfoq20throughput}, or by minimizing the number of iterations  \cite{lebar23}. To our knowledge, \cite{Chiu23JSAC} is the only prior work that explicitly designed mixing matrices for minimizing energy consumption. However, \cite{Chiu23JSAC} only considered the total energy consumption, but this work considers the energy consumption at the busiest node. 
Our design is based on an objective function that generalizes the spectral gap objective \cite{Vogels22} to random mixing matrices. Spectral gap is an important parameter for capturing the impact of topology on the convergence rate of decentralized learning~\cite{Lian17NIPS,
Neglia19INFOCOM}. 
Even if recent works identified some other parameters through which the topology can impact the convergence rate, such as the effective number of neighbors~\cite{Vogels22} and the neighborhood heterogeneity~\cite{lebar23}, their results did not invalidate the impact of spectral gap and just pointed out additional factors. 
\fi

\subsection{Summary of Contributions}
We study the design of mixing matrix in decentralized learning with the following contributions: \looseness=0

1) Instead of considering the total energy consumption as in \cite{Chiu23JSAC}, our design aims at minimizing the energy consumption at the busiest node, leading to a more balanced load. 

 2) Instead of using a heuristic objective as in \cite{MATCHA19} or a partially justified objective as in \cite{Chiu23JSAC}, we use a fully theoretically-justified design objective, which enables a new approach for mixing matrix design based on graph sparsification. 
 
 3) Based on the new approach, we propose an algorithm with guaranteed performance for a special case and a greedy heuristic for the general case. Our solution achieves $54\%$--$76\%$ lower energy consumption at the busiest node while producing a model of the same quality as the best-performing benchmark in simulations based on real topology and dataset. 

\textbf{Roadmap.} Section~\ref{sec:Background and Formulation} formulates our problem, Section~\ref{sec:Proposed Solution} presents the proposed solution, Section~\ref{sec:Performance Evaluation} evaluates it against benchmarks, and Section~\ref{sec:Conclusion} concludes the paper. 
\if\thisismainpaper1
\textbf{Proofs and additional evaluation results are provided in \cite{Xusheng23:report}.}
\else
Proofs and additional evaluation results are provided in the appendix. 
\fi

\section{Background and Problem Formulation}\label{sec:Background and Formulation}


\subsection{Decentralized Learning Algorithm}

Consider a network of $m$ nodes connected through a \emph{base topology} $G=(V, E)$ ($|V|=m$), where $E$ defines the pairs of nodes that can directly communicate. 
Each node $i\in V$ has a local objective function $F_i(\bm{x})$ that depends on the parameter vector  $\bm{x}\in \mathbb{R}^d$ and its local dataset. The goal is to minimize the global objective function 
$
    F(\boldsymbol{x}) := \frac{1}{m}\sum_{i=1}^{m} F_{i}(\boldsymbol{x}).
$

We consider a state-of-the-art decentralized learning algorithm called D-PSGD~\cite{Lian17NIPS}. Let $\bm{x}_i^{(k)}$ ($k\geq 1$) denote the parameter vector at node $i$ after $k-1$ iterations and $g(\bm{x}_i^{(k)}; \xi_i^{(k)})$ the stochastic gradient computed in iteration $k$. D-PSGD runs the following update in parallel at each node $i$:\looseness=-1
\scale{
\begin{align}\label{eq:DecenSGD}
    \boldsymbol{x}^{(k+1)}_i = \sum_{j=1}^{m}W^{(k)}_{ij}(\boldsymbol{x}^{(k)}_j - \eta g(\boldsymbol{x}^{(k)}_j; \xi^{(k)}_j)),
\end{align}
}where $\bm{W}^{(k)}=(W^{(k)}_{ij})_{i,j=1}^m$ is the $m\times m$ \emph{mixing matrix} in iteration $k$, and $\eta>0$ is the learning rate. 
To be consistent with the base topology, $W^{(k)}_{ij}\neq 0$ only if $(i,j)\in E$. 
Finally we let $\overline{\bm{x}}^{(k)}:={1\over m}\sum_{i=1}^m \bm{x}^{(k)}_i$.

The convergence of this algorithm is guaranteed under the following assumptions:
\begin{enumerate}[(1)]
    \item Each local objective function $F_i(\boldsymbol{x})$ is $l$-Lipschitz smooth, i.e.,\footnote{For a vector $\bm{a}$, $\|\bm{a}\|$ denotes the $\ell$-2 norm. For a matrix $\bm{A}$, $\|\bm{A}\|$ denotes the spectral norm, and $\|\bm{A}\|_F$ denotes the Frobenius norm.} $\|\nabla F_i(\bm{x})-\nabla F_i(\bm{x}')\|\leq l\|\bm{x}-\bm{x}'\|,\: \forall i\in V$.
\item There exist constants $M_1, \hat{\sigma}$ such that 
    \scale{
    $\hspace{-0em} {1\over m}\sum_{i\in V}\E[\|g(\bm{x}_i;\xi_i)-\nabla F_i(\bm{x}_i) \|^2] \leq \hat{\sigma}^2 + {M_1\over m}\sum_{i\in V}\|\nabla F(\bm{x}_i)\|^2$, 
    $\forall \bm{x}_1,\ldots,\bm{x}_m \in \mathbb{R}^d$.
    }
    \item There exist constants $M_2, \hat{\zeta}$ such that \scale{
    ${1\over m}\sum_{i\in V}\|\nabla F_i(\bm{x})\|^2\leq \hat{\zeta}^2 + M_2\|\nabla F(\bm{x}) \|^2, \forall \bm{x} \in \mathbb{R}^d$.
    }
\end{enumerate}

\begin{theorem} \label{thm:new convergence bound}
    Let $\bm{J}:={1\over m}\bm{1} \bm{1}^\top$ and let $\bm{W}$ be a random symmetric matrix such that each row/column in $\bm{W}$ sums to one. 
    Let $\rho:= \|\E[\bm{W}^\top\bm{W}]-\bm{J}\|$.
    Under assumptions (1)--(3), if each mixing matrix $\bm{W}^{(k)}$ is an i.i.d. copy of $\bm{W}$ 
    and $\rho < 1$, then D-PSGD can achieve $ \frac{1}{K} \sum_{k=1}^K \mathbbm{E}[\|\nabla F(\boldsymbol{\overline{x}}^{(k)})\|^2]\leq \epsilon$ for any given $\epsilon>0$ ($\overline{\bm{x}}^{(k)}:={1\over m}\sum_{i=1}^m \bm{x}^{(k)}_i$) when the number of iterations reaches
\scale{
\begin{align}
K(\rho)&:= l(F(\overline{\bm{x}}^{(1)})-F_{\inf}) \nonumber\\
&\hspace{-2.75em}   \cdot \hspace{-.15em}O\hspace{-.25em}\left(\hspace{-.25em}{\hat{\sigma}^2\over m\epsilon^2} \hspace{-.15em}+\hspace{-.15em}{\hat{\zeta}\sqrt{M_1+1}+\hat{\sigma}\sqrt{1-\rho}\over (1-\rho) \epsilon^{3/2}} \hspace{-.15em}+\hspace{-.15em} {\sqrt{(M_2+1)(M_1+1)}\over (1-\rho)\epsilon} \hspace{-.05em}\right). \nonumber 
\end{align}
}
\end{theorem}
\emph{Remark:} The required number of iterations $K(\rho)$ depends on the mixing matrix only through the parameter $\rho$: the smaller $\rho$, the fewer iterations are needed. 

The proof of Theorem~\ref{thm:new convergence bound} is based on \cite[Theorem~2]{Koloskova20ICML} and included in Appendix~\ref{subsec:Supporting Proofs}.

\subsection{Mixing Matrix }

As node $i$ needs to send its parameter vector to node $j$ in iteration $k$ only if $W^{(k)}_{ij}\neq 0$, we can control the communications by designing the mixing matrix $\bm{W}^{(k)}$. To this end, we use \looseness=-1
$    \boldsymbol{W}^{(k)} := \boldsymbol{I} -  \boldsymbol{L}^{(k)},$
where $\boldsymbol{L}^{(k)}$ is the weighted Laplacian matrix~\cite{Bollobas13} of the topology $G^{(k)}=(V,E^{(k)})$ \emph{activated} in iteration $k$. 
Given the incidence matrix\footnote{Matrix $\bm{B}$ is a $|V| \times |E|$ matrix, defined as $B_{ij}=1$ if link $e_j$ starts at node $i$ (under arbitrary link orientation), $-1$ if $e_j$ ends at $i$, and $0$ otherwise.} $\bm{B}$ of the base topology $G$ and a vector of \emph{link weights} $\bm{\alpha}^{(k)}$, the Laplacian matrix $\bm{L}^{(k)}$ is given by
$\boldsymbol{L}^{(k)}=\boldsymbol{B} \operatorname{diag}(\boldsymbol{\alpha}^{(k)}) \boldsymbol{B}^{T}.$
The above reduces the mixing matrix design problem to a problem of designing the link weights $\bm{\alpha}^{(k)}$, where {a link $(i,j)\in E$ will be activated in iteration $k$} \emph{if and only if $\alpha_{(i,j)}^{(k)}\neq 0$}. 
This construction guarantees that $\bm{W}^{(k)}$ is symmetric with each row/column summing up to one. 

\subsection{Cost Model}

We use $c(\bm{\alpha}^{(k)}):=(c_i(\bm{\alpha}^{(k)}))_{i=1}^m$ to denote the cost vector in an iteration when the link weight vector is $\bm{\alpha}^{(k)}$. We focus on the energy consumption at each node $i$, which contains two parts: (i) computation energy $c_i^a$ for computing the local stochastic gradient and the local aggregation, and (ii) communication energy $c_{ij}^b$ for sending the updated local parameter vector from node $i$ to node $j$. Then the energy consumption at node $i$ in iteration $k$ is modeled as
\looseness=0 
\scale{
\begin{align}\label{eq:cost definition}
c_i(\bm{\alpha}^{(k)}) := c^a_i + \sum_{j: (i,j)\in E} c^b_{ij} \mathbbm{1}(\alpha_{(i,j)}^{(k)}\neq 0),
\end{align}
}where $\mathbbm{1}(\cdot)$ denotes the indicator function. This cost function models the basic scenario where all communications are point-to-point and independent.  Other scenarios are left to future work. \looseness=-1

\subsection{Optimization Framework}

To trade off between the cost per iteration and the convergence rate, we adopt a bi-level optimization framework:

\noindent \textbf{Lower-level optimization:}~design link weights $\bm{\alpha}$ to maximize the convergence rate (by minimizing $\rho_\Delta$) under a given budget $\Delta$ on the maximum cost per node in each iteration, which results in a required number of iterations of $K(\rho_\Delta)$. 

\noindent \textbf{Upper-level optimization:}~design $\Delta$ to minimize the total maximum cost per node ${\Delta}\cdot K(\rho_\Delta)$.

\section{Mixing Matrix Design via Graph Sparsification}\label{sec:Proposed Solution}

As the upper-level optimization only involves one scalar decision variable, we will focus on the lower-level optimization.\looseness=-1

\subsection{Simplified Objective}



Theorem~\ref{thm:new convergence bound} implies that the lower-level optimization should minimize $\rho$. While it is possible to formulate this minimization in terms of the link weights $\bm{\alpha}$, the resulting optimization problem, with a form similar to \cite[(18)]{Chiu23JSAC}, will be intractable due to the presence of non-linear matrix inequality constraint. We thus further simplify the objective as follows. 

\begin{lemma}\label{lem:rho upper-bound}
For any mixing matrix $\bm{W}:=\bm{I}-\bm{L}$, where $\bm{L}$ is a randomized Laplacian matrix, 
\scale{
\begin{align}
\rho &\leq \E\left[\|\bm{I}-\bm{L}-\bm{J}\|^2\right] \label{eq:rho upper-bound matrix}\\
&= \E\left[\max\left((1-\lambda_2(\bm{L}))^2, (1-\lambda_m(\bm{L}))^2\right) \right], \label{eq:rho upper-bound eigen}
\end{align}
}
where $\lambda_i(\bm{L})$ denotes the $i$-th smallest eigenvalue of $\bm{L}$. 
\end{lemma}

By Lemma~\ref{lem:rho upper-bound}, we relax the objective of the lower-level optimization to designing a randomized $\bm{\alpha}$ by solving 
\scale{
\begin{subequations}\label{eq:min rho bound}
    \begin{align}
\min_{\bm{\alpha}} &\quad \E\left[\|\bm{I}-\bm{B} \diag(\bm{\alpha})\bm{B}^\top-\bm{J}\|^2\right]\\
\mbox{s.t. } 
&\quad \E[c_i(\bm{\alpha})] \leq \Delta, ~~~\forall i\in V. \label{eq: budget constraint}
    \end{align}
\end{subequations}
}

\subsection{Idea on Leveraging Graph Sparsification}

We propose to solve the relaxed lower-level optimization \eqref{eq:min rho bound} based on graph (spectral) sparsification. 
First, we compute the optimal link weight vector $\bm{\alpha}'$ without the budget constraint \eqref{eq: budget constraint} by solving the following optimization:
\scale{
\begin{subequations}\label{eq:min rho wo cost}
\begin{align}
& \min_{\bm{\alpha}}\:  \tilde{\rho} \label{wo cost:obj}\\
\mbox{s.t. }
& - \tilde{\rho}\bm{I} \preceq \bm{I} - \bm{B} \diag(\bm{\alpha}) \bm{B}^\top - \bm{J} \preceq  \tilde{\rho}\bm{I}. \label{wo cost:matrix}
\end{align}
\end{subequations}
}
Constraint \eqref{wo cost:matrix} ensures $\tilde{\rho} = \|\bm{I}-\bm{B}\diag(\bm{\alpha})\bm{B}^\top -\bm{J}\|$ at the optimum, i.e., $\bm{\alpha}'$ minimizes $\|\bm{I}-\bm{B} \diag(\bm{\alpha})\bm{B}^\top -\bm{J}\|$. 
Optimization \eqref{eq:min rho wo cost} is a semi-definite programming (SDP) problem that can be solved in polynomial time by existing algorithms \cite{Jiang20FOCS}. 
The vector $\bm{\alpha}'$ establishes a lower bound on the relaxed objective: 
if 
$\bm{\alpha}^*$ is the optimal randomized solution for \eqref{eq:min rho bound}, 
then 
$\E\left[\|\bm{I}-\bm{B} \diag(\bm{\alpha}^*)\bm{B}^\top-\bm{J}\|^2\right] 
\geq \|\bm{I}-\bm{B} \diag(\bm{\alpha}')\bm{B}^\top -\bm{J} \|^2$. \looseness=-1 

Then, we use a graph sparsification algorithm to sparsify the weighted graph with link weights $\bm{\alpha}'$ to satisfy the budget constraint. 
As $\|\bm{I}-\bm{B} \diag(\bm{\alpha}')\bm{B}^\top -\bm{J} \|^2$ $= \max\big((1-\lambda_2(\bm{B} \diag(\bm{\alpha}')\bm{B}^\top))^2,$ $(1-\lambda_m(\bm{B} \diag(\bm{\alpha}')\bm{B}^\top))^2\big)$, and graph sparsification aims at preserving the original eigenvalues  $\left(\lambda_i(\bm{B} \diag(\bm{\alpha}')\bm{B}^\top)\right)_{i=1}^m$~\cite{Spars08}, the sparsified link weight vector $\bm{\alpha}_s$ is expected to achieve an objective value $\|\bm{I}-\bm{B} \diag(\bm{\alpha}_s)\bm{B}^\top -\bm{J} \|^2$ that approximates the optimal for \eqref{eq:min rho bound}. \looseness=-1

\subsection{Algorithm Design}

We now apply the above idea to develop algorithms for mixing matrix design.  

\subsubsection{Ramanujan-Graph-based Design for a Special Case}\label{subsubsec:Ramanujan}

Consider the special case when the base topology $G$ is a complete graph and all transmissions by a node have the same cost, i.e., $c_{ij}^b\equiv c_i^b$ for all $j$ such that $(i,j)\in E$. Let $d:=\min_{i\in V}\lfloor (\Delta-c_i^a)/c_i^b \rfloor$. Then any graph with degrees bounded by $d$ satisfies the budget constraint. 
The complete graph has an ideal sparsifier known as \emph{Ramanujan graph}. 
A $d$-regular graph $H$ is a Ramanujan graph if all the non-zero eigenvalues of its Laplacian matrix $\bm{L}_H$ lie between $d-2\sqrt{d-1}$ and $d+2\sqrt{d-1}$ \cite{hoory06}.
By assigning weight ${1}/{d}$ to every link of a Ramanujan graph $H$, we obtain a weighted graph $H'$, whose Laplacian $\bm{L}_{H'}$ satisfies $\lambda_1(\bm{L}_{H'}) = 0$ and \looseness=0
\scale{
\begin{align}
1-{2\sqrt{d-1}\over d}\leq \lambda_2(\bm{L}_{H'})\le\cdots \leq 
    \lambda_m(\bm{L}_{H'}) \le 1+ {2\sqrt{d-1}\over d}. \nonumber
\end{align}
}By Lemma~\ref{lem:rho upper-bound}, the deterministic mixing matrix  $\bm{W}_{H'} := \bm{I}-\bm{L}_{H'}$ 
achieves a $\rho$-value $\rho_{H'}$ that satisfies
\scale{
\begin{align}
\rho_{H'} 
&\leq \max\left((1-\lambda_2(\bm{L}_{H'}))^2, (1-\lambda_m(\bm{L}_{H'}))^2 \right) \nonumber\\
&\leq {4(d-1)\over d^2}
= O\left({1\over d}\right) = O\left({1\over \Delta}\right).\nonumber
\end{align}
}Ramanujan graphs can be easily constructed by drawing random $d$-regular graphs until satisfying the Ramanujan definition \cite{Friedman03}. By the result of \cite{KimVu03, PuWormald17}, for $d =o(\sqrt{m})$, we can generate random $d$-regular graphs in polynomial time. Thus, the above method can efficiently construct a deterministic mixing matrix with guaranteed performance in solving the lower-level optimization for a given budget $\Delta$ such that $d =o(\sqrt{m})$. 

\subsubsection{Intractability for General Case}
We will see that finding a feasible graph sparsifier is computationally hard in the general case.  
To facilitate the discussions, assume for all $i$ and $j$, $c_i^a = 0$ and $c_{ij}^b = 1$, so the budget constraint \eqref{eq: budget constraint} translates to a maximum degree constraint.
For a general base topology $G$ and any fixed $\Delta$, 
it is not clear whether there exists a feasible Laplacian $\bm{L}$ satisfying the maximum degree constraint (recall that for $\bm{L}$ to be feasible, its convergence parameter $\rho$ for $\bm{W}=\bm{I}-\bm{L}$ must be strictly less than $1$). For example, a subgraph of a star graph needs $m-1$ edges incident to the center to remain connected, and if $\Delta< m-1$, then any subgraph with maximum degree at most $\Delta$ is disconnected, which implies $\lambda_2(\bm{L}) = 0$.
Hence, through similar computations as in \eqref{eq:rho upper-bound matrix}--\eqref{eq:rho upper-bound eigen}, 
for a deterministic $\bm{L}$,
\begin{align}
\rho(\bm{W}) = \|\bm{I}-\bm{L}-\bm{J}\|^2 \ge (1 - \lambda_2(\bm{L}))^2 = 1. \label{eq:necessary condition}
\end{align}
Moreover, in general, the task of determining whether there exists a feasible $\bm{L}$  satisfying the maximum degree constraint is NP-hard
because deciding the existence of a connected spanning subgraph with maximum degree no more than $\Delta$ is NP-hard. 
The following theorem provides NP-hardness for a slightly more general problem; its proof is provided in Appendix~\ref{subsec:Supporting Proofs}.
\begin{theorem}\label{thm:nphard}
Given a graph $G_{cand}=(V,E_{cand})$, and a degree constraint $d_i$ for each vertex $v_i\in V$,
then it is NP-hard to decide the existence of $E\subseteq E_{cand}$ such that $G=(V,E)$ is a connected graph and $\deg_G(v_i)\le d_i$ for all vertices $v_i\in V$.
\end{theorem}
Finding a feasible graph sparsifier under degree constraints is equivalent to finding a connected spanning subgraph under the same constraints, as 
\[\rho(\bm{W})= \max\left((1-\lambda_2(\bm{L}))^2, (1-\lambda_m(\bm{L}))^2\right)\]
and we can set the link weights to be small enough so that $\lambda_m(\bm{L})<2$, under which $\rho(\bm{W})<1$ if and only if $\lambda_2(\bm{L})>0$. 
Therefore, for a general base topology with general costs and a general budget constraint, 
it is algorithmically intractable to find a feasible graph sparsifier.

\subsubsection{Greedy Heuristic for General Case}

For the general case, ideally we want to sparsify a weighted graph $G'$ with link weights $\bm{\alpha}'$ such that the sparsified graph with link weights $\bm{\alpha}_s$ will approximate the eigenvalues of $G'$ while satisfying the constraint $c_i(\bm{\alpha}_s)\leq \Delta$ for each $i\in V$. While this remains an NP-hard problem for general graphs, we propose a greedy heuristic based on the intuition that the \emph{importance of a link is reflected in its {absolute} weight}. Specifically, we will find the link $(i,j)$ with the minimum absolute weight according to the solution to \eqref{eq:min rho wo cost} such that the cost for either node $i$ or node $j$ exceeds the budget $\Delta$, set $\alpha_{(i,j)}=0$, and then find the next link by re-solving \eqref{eq:min rho wo cost} under this additional constraint,  
until either all the nodes satisfy the budget or the graph becomes disconnected;
in the latter case, the algorithm reports failure to find a sparsifier under budget $\Delta$.
\looseness=-1

\section{Performance Evaluation}\label{sec:Performance Evaluation}

We evaluate the proposed solution for the general case based on a real dataset and the topology of a real wireless network. We defer the evaluation in the special case to 
\if\thisismainpaper1
\cite{Xusheng23:report}.
\else
Appendix~\ref{subsec:Additional Evaluation}. 
\fi

\emph{Experiment setting:} 
We consider training for image classification based on CIFAR-10, which consists of 60,000 color images in 10 classes. We train the ResNet-50 model over its training dataset with 50,000 images, and then test the trained model over the testing dataset with 10,000 images. 
We use the topology of Roofnet~\cite{Roofnet} at data rate 1 Mbps as the base topology, which contains 33 nodes and 187 links.
To evaluate the cost, we set the computation energy as $c^a_i = 0.0003342$ (Wh) and the communication energy as $c^b_{ij} = 0.0138$ (Wh) based on our parameters and the parameters from \cite{Carbon2021}\footnote{Our model size is $S=2.3$MB, batch size is 32, and processing speed is 8ms per sample. Assuming 1Mbps links and TX2 as the hardware, whose power is 4.7W during computation and 1.35W during communication~\cite{Carbon2021}, we estimate the computation energy by $c^a_i = 4.7*32*0.008/3600\approx 0.0003342$Wh, and the communication energy with each neighbor by $c^b_{ij}=2*1.35*S*8/1/3600$Wh, where the multiplication by 2 is because this testbed uses WiFi, which is half-duplex.}.
Following \cite{Chiu23JSAC}, we set the learning rate as 0.8 at the beginning and reduce it by 10X after 100, 150, 180, 200 epochs, and the mini-batch size to 32. 

\emph{Benchmarks:}
We compare the proposed solution with with four benchmarks: `Vanilla D-PSGD' \cite{Lian17NIPS} where all the neighbors communicate in all the iterations, `Periodic' where all the neighbors communicate periodically, `MATCHA'~\cite{MATCHA19} which was designed to minimize training time, and Algorithm~1 in \cite{Chiu23JSAC} (`Greedy total' \if\thisismainpaper0or `Gt' \fi) for the cost model \eqref{eq:cost definition} which was designed to minimize the total energy consumption\footnote{While the final solution in \cite{Chiu23JSAC} was randomized over a set of mixing matrices, we only use the deterministic design by Algorithm~1 for a fair comparison, as the same randomization can be applied to the proposed solution.\looseness=-1}. 
In `Vanilla D-PSGD', `Periodic', and `MATCHA', identical weights are assigned to every activated link, whereas in `Greedy total' and the proposed algorithm, heterogeneous link weights are parts of the designs.
We first tune MATCHA to minimize its loss at convergence, and then tune the other benchmarks to activate the same number of links on the average. We evaluate two versions of the proposed algorithm (`Greedy per-node' \if\thisismainpaper0or `Gp'\fi): one with the same maximum energy consumption per node as the best-performing benchmark (leading to a budget that amounts to $55\%$ of maximum degree) and the other with the same accuracy as the best-performing benchmark at convergence (leading to a budget that amounts to $25\%$ of maximum degree). \looseness=-1

\begin{figure}[t!]
    \centering
    \centerline{\mbox{\includegraphics[width=1.0\linewidth]{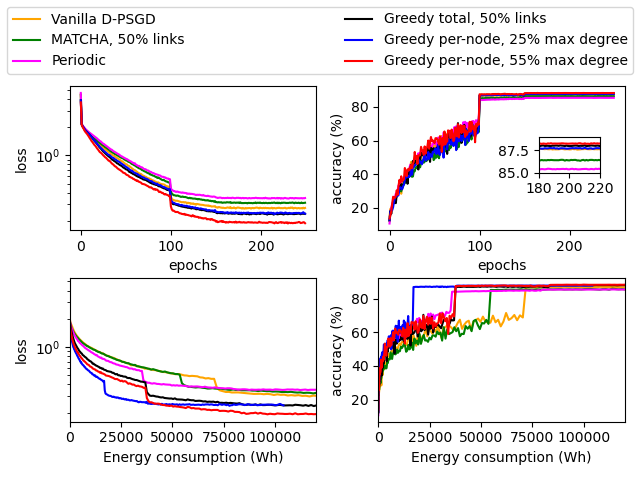}}}
    \vspace{-1em}
    \caption{Training loss and testing accuracy for decentralized learning over Roofnet. }
    \label{fig:results_general}
    \vspace{-1em}
\end{figure}

\if\thisismainpaper0
\begin{table}[htbp]
  \caption{Stats at epoch 200.}\label{tab:Stats at last epoch}
  \small
  \begin{tabular}{lcccc}
    \toprule
    \textbf{Method} & \textbf{Loss} & \textbf{Acc.}& \textbf{Per-node Ene.} & \textbf{Total Ene.} \\
    \midrule
    Vanilla & 0.277 & 87.6\% & 280kWh & 4980kWh \\
    Periodic & 0.350 & 85.4\% & 140kWh & 2490kWh \\
    MATCHA & 0.313 & 86.4\% & 213kWh & 2490kWh \\
    Gt, 50\% & 0.236 & 88.0\% & 147kWh & 2477kWh\\
Gp, 25\% & 0.244 & 87.7\% & 67kWh & 1465kWh\\
Gp, 55\% & 0.192 & 88.2\% & 147kWh & 3169kWh\\
    \bottomrule
  \end{tabular}
\end{table}
\fi

\emph{Results:} Fig.~\ref{fig:results_general} shows the loss and accuracy of the trained model, with respect to both the epochs and the maximum energy consumption per node. We see that: (i) instead of activating all the links as in `Vanilla D-PSGD', it is possible to activate fewer (weighted) links without degrading the quality of the trained model; (ii) different ways of selecting the links to activate lead to different quality-cost tradeoffs; (iii) the algorithm designed to optimize the total energy consumption (`Greedy total') performs the best among the benchmarks; (iv) however, by balancing the energy consumption across nodes, the proposed algorithm (`Greedy per-node') can achieve either a better loss/accuracy at the same maximum energy consumption per node, or a lower maximum energy consumption per node at the same loss and accuracy. In particular, the proposed algorithm (at $25\%$ maximum degree) 
can save $54\%$ energy at the busiest node compared to the best-performing benchmark (`Greedy total') and $76\%$ compared to `Vanilla D-PSGD', while producing a model of the same quality. 
Meanwhile, the proposed algorithm also saves $41$--$71\%$ of the total energy consumption compared to the benchmarks, as shown in 
\if\thisismainpaper1
\cite[Table~1]{Xusheng23:report}.
\else
Table~\ref{tab:Stats at last epoch}.
\fi
\looseness=-1

\section{Conclusion}\label{sec:Conclusion}

Based on an explicit characterization of how the mixing matrix affects the convergence rate in decentralized learning, we proposed a bi-level optimization for mixing matrix design, with the lower level solved by graph sparsification. This enabled us to develop a solution with guaranteed performance for a special case and a heuristic for the general case. Our solution greatly reduced the energy consumption at the busiest node while maintaining the quality of the trained model.  \looseness=-1

\vfill\pagebreak

\bibliographystyle{IEEEbib}
\bibliography{references}

\if\thisismainpaper0
\appendix

\section{Supporting Proofs}\label{subsec:Supporting Proofs}

\subsection{Proof of Theorem~\ref{thm:new convergence bound}} 
We first recall the following result from \cite{Koloskova20ICML}.
\begin{theorem}\cite[Theorem~2]{Koloskova20ICML} \label{thm:K21 convergence bound}
Let $\bm{J}:={1\over m}\bm{1} \bm{1}^\top$. Under assumptions (1)--(3), if there exist a constant $p\in (0,1]$ such that the mixing matrices $\{\bm{W}^{(k)}\}_{k=1}^K$, each being symmetric with each row/column summing to one\footnote{Originally, \cite[Theorem~2]{Koloskova20ICML} had a stronger assumption that each mixing matrix is doubly stochastic, but we have verified that it suffices to have each row/column summing to one.}, satisfy
\scale{
\begin{align}\label{eq:condition on p}
\hspace{-.5em}    \E[\|\bm{X}\bm{W}^{(k)} - \bm{X}\bm{J}\|_F^2] \leq (1-p) \|\bm{X}-\bm{X}\bm{J}\|_F^2, 
\end{align}
}for all $\bm{X}:= [\bm{x}_1,\ldots,\bm{x}_m]$ and integer $k\geq 1$, then D-PSGD can achieve $ \frac{1}{K} \sum_{k=1}^K \mathbbm{E}[\|\nabla F(\boldsymbol{\overline{x}}^k)\|^2]\leq \epsilon_0$ for any given $\epsilon_0>0$ ($\overline{\bm{x}}^{(k)}:={1\over m}\sum_{i=1}^m \bm{x}^{(k)}_i$) when the number of iterations reaches
\scale{
\begin{align}
K(p)&:= l(F(\overline{\bm{x}}^{(1)})-F_{\inf}) \nonumber\\
&\hspace{-2.75em}   \cdot \hspace{-.15em}O\hspace{-.25em}\left(\hspace{-.25em}{\hat{\sigma}^2\over m\epsilon_0^2} \hspace{-.15em}+\hspace{-.15em}{\hat{\zeta}\sqrt{M_1+1}+\hat{\sigma}\sqrt{p}\over p \epsilon_0^{3/2}} \hspace{-.15em}+\hspace{-.15em} {\sqrt{(M_2+1)(M_1+1)}\over p\epsilon_0} \hspace{-.05em}\right). \nonumber 
\end{align}
}
\end{theorem}
\emph{Remark:} Originally, \cite[Theorem~2]{Koloskova20ICML} only mandates \eqref{eq:condition on p} for the product of $\tau$ mixing matrices, but we consider the case of $\tau = 1$ for the tractability of mixing matrix design.

Since $\bm{W}^{(1)}, \dots, \bm{W}^{(K)}$ are i.i.d. copies of a random matrix $\bm{W}$ in our case, we first rewrite \eqref{eq:condition on p} as
\looseness=-1
\scale{
\begin{align}
    p := \min_{\bm{X}\neq \bm{0}} \left(1-{\E[\|\bm{X}(\bm{W}-\bm{J})\|_F^2]\over \|\bm{X}(\bm{I}-\bm{J})\|_F^2}\right).\label{eq:p-maximization objective}
\end{align}
}Yet \eqref{eq:p-maximization objective} is not an explicit function of the mixing matrix, so in the next lemma, we relate it to an equivalent quantity that is an explicit function of $\bm{W}$ and thus easier to handle. 

\begin{lemma}\label{lem:p=1-rho}
For any randomized mixing matrix $\bm{W}$ that is symmetric with every row/column summing to one, $p$ defined in \eqref{eq:p-maximization objective} satisfies $p = 1-\rho$ for $\rho := \|\E[\bm{W}^\top\bm{W}]-\bm{J}\|$. 
\end{lemma} 

Theorem~\ref{thm:new convergence bound} follows from Theorem~\ref{thm:K21 convergence bound} and Lemma~\ref{lem:p=1-rho},
so it remains to prove Lemma~\ref{lem:p=1-rho}. 

\begin{proof}[Proof of Lemma~\ref{lem:p=1-rho}]
One direction $p \le 1 - \rho$ was proved in \cite{Chiu23JSAC}, and we will prove that $p \ge 1 - \rho$, or equivalently $1-p\le \rho$.
For this, we rely on the following fact (see \cite[Lemma~1]{MATCHA19}\footnote{Although \cite[Lemma~1]{MATCHA19} originally assumed $\bm{W}$ to be doubly stochastic, we have verified that having each row/column summing to one is sufficient.}): for any matrix $\bm{A}\in \mathbb{R}^{d\times m}$, 
\begin{equation}
    \label{eq:lem31a}
    \E \left[ \|\bm{A}(\bm{W}-\bm{J})\|_F^2 \right] \le \rho \|\bm{A} \|_F^2.
\end{equation}
We fix a matrix $\bm{X} \neq 0$.
Now set $\bm{A} = \bm{X}(\bm{I}-\bm{J})$, and
\eqref{eq:lem31a} yields that 
\begin{equation}
    \label{eq:lem31b}
    \frac{ \E \left[ \|\bm{X}(\bm{I}-\bm{J})(\bm{W}-\bm{J})\|_F^2 \right]}{ \|\bm{X}(\bm{I}-\bm{J})\|_F^2}\le \rho.
\end{equation}
Note that for our choice of matrix $\bm{W}$, $\bm{WJ}=\bm{JW}=\bm{JJ}=\bm{J}$. Hence, 
\begin{equation}
    \label{eq:lem31c}
    \bm{X}(\bm{I}-\bm{J})(\bm{W}-\bm{J})
    = \bm{X} (\bm{W} - \bm{J} - \bm{J} +\bm{J}) 
    = \bm{X} (\bm{W} - \bm{J}).
\end{equation}
Thus, by \eqref{eq:lem31b} and \eqref{eq:lem31c}, 
we establish that 
\begin{equation}
    \label{eq:lem31d}
    {\E[\|\bm{X}(\bm{W}-\bm{J})\|_F^2]\over \|\bm{X}(\bm{I}-\bm{J})\|_F^2} \le \rho.
\end{equation}
Since $\bm{X}$ is an arbitrary nonzero matrix, it follows from \eqref{eq:p-maximization objective} that 
$1-p\le \rho$.
\end{proof}

\subsection{Proof of  Lemma~\ref{lem:rho upper-bound}}

\begin{proof}[Proof of Lemma~\ref{lem:rho upper-bound}]
As the spectral norm is convex, Jensen's inequality implies\looseness=-1 
\begin{equation}
    \label{eq:Jensenmtx}
     \rho:=\|\E[\bm{W}^\top\bm{W}]-\bm{J}\|
     \le 
      \E\|\bm{W}^\top\bm{W}-\bm{J}\|.
\end{equation}
For a given Laplacian matrix $\bm{L}$, 
\begin{align}
\|\bm{W}^{\top}\bm{W}-\bm{J}\| = \|\bm{W} - \bm{J}\|^2 = \|\bm{I}-\bm{L}-\bm{J}\|^2, \label{eq:rho proof - 1}    
\end{align}
where the first ``$=$'' is because the eigenvalues of $\bm{W}^{\top}\bm{W}-\bm{J}$ are squares of the eigenvalues of $\bm{W}-\bm{J}$ as shown by eigenvalue decomposition. 
By Lemma~IV.2 in \cite{Chiu23JSAC}, 
\begin{align}\label{eq:rho proof - 2}
\|\bm{I}-\bm{L}-\bm{J}\| = \max\{1 - \lambda_{2}(\bm{L}),\: \lambda_{m}(\bm{L}) - 1\}, 
\end{align}
Combining \eqref{eq:Jensenmtx}--\eqref{eq:rho proof - 2} yields \eqref{eq:rho upper-bound eigen}.
\end{proof}

\subsection{Proof of Theorem~\ref{thm:nphard}}

\begin{proof}[Proof of Theorem~\ref{thm:nphard}]
It suffices to reduce from the Hamiltonian path problem~\cite{Garey90} to our problem. 
Recall the Hamiltonian path problem for a given graph $G = (V, E)$ is a problem of determining the existence of a path $P\subseteq E$ visiting each vertex $v\in V$ exactly once; 
such path is referred to as \emph{ Hamiltonian path}.

First, we construct an input graph $G_{cand}$ 
for each instance  $G_0=(V_0, E_0)$ of the Hamiltonian path problem. 
Let $V=V_0, E_{cand} = E_0$, and
$d_i=2$ for all $v_i\in V_0$.

Second, 
we show there exists $E\subseteq E_{cand}$ such that 
if $H = (V_0, E)$ then 
$deg_{H}(v_i) \le d_i$ for each vertex $v_i \in V_0$ and
$\lambda_2(\bm{L}_{H}) > 0$ if and only if $G_0$ has a Hamiltonian path.

If there exists a Hamiltonian path $P\subseteq E_0$, then since the degree of each vertex in $H=(V_0, P)$ would be less than or equal to 2, $H$ satisfies the constraint $deg_H(v_i)\leq 2$; 
moreover, $\lambda_2(\bm{L}_H) >0$ since $H$ is connected.

Conversely, suppose $H=(V_0, E)$ satisfies $E \subseteq E_{cand}$, $deg_{H}(v_i) \le d_i$ for each vertex $v_i \in V_0$ and
$\lambda_2(\bm{L}_{H}) > 0$.
Then apparently $H$ is connected, and
a connected graph with degrees less than or equal to 2 can only be a path or a cycle. A path must be a Hamiltonian path in $G_0$ since it connects all the vertices; similarly, a cycle must contain a Hamiltonian path. 

The proof is now complete as the Hamiltonian path problem is NP-complete~\cite{Garey90}.
\end{proof}

\section{Additional Evaluation Results}\label{subsec:Additional Evaluation}

\begin{figure}[t!]
    \centering
    \centerline{\mbox{\includegraphics[width=1.0\linewidth]{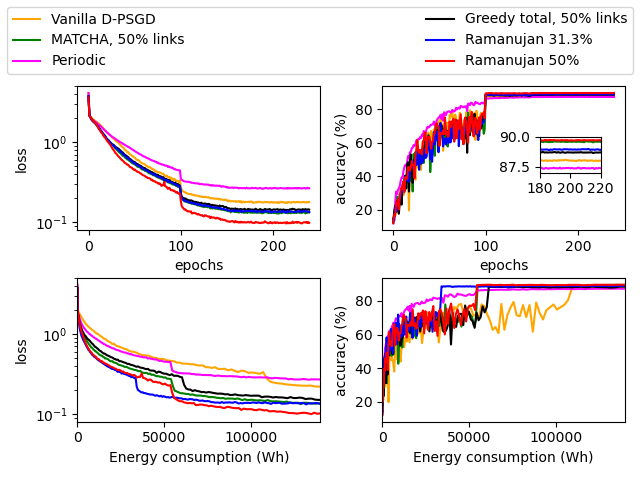}}}
    \vspace{-1em}
    \caption{Training loss and testing accuracy for decentralized learning over a complete graph. }
    \label{fig:results complete}
    \vspace{-1em}
\end{figure}

In addition to the evaluation of the general case in Section~\ref{sec:Performance Evaluation}, we also evaluate the special case of a fully-connected base topology. We use the same experiment setting and benchmarks as in Section~\ref{sec:Performance Evaluation}, except that the base topology is a $33$-node complete graph. The proposed solution in this case is the Ramanujan-graph-based design in Section~\ref{subsubsec:Ramanujan}. We still evaluate two versions of this solution: one with the same maximum energy consumption per node as the best-performing benchmark (leading to a budget that amounts to $50\%$ of node degree in the base topology) 
and the other with an accuracy no worse than vanilla D-PSGD (leading to a budget that amounts to $31.3\%$ of node degree). 

The results in Fig.~\ref{fig:results complete} show that: (i) similar to Fig.~\ref{fig:results_general}, careful selection of the links to activate can notably improve the quality-cost tradeoff in decentralized learning; (ii) however, the best-performing benchmark under fully-connected base topology becomes `MATCHA' even if it was designed for a different objective \cite{MATCHA19}; (iii) nevertheless, by intentionally optimizing a parameter \eqref{eq:rho upper-bound eigen} controlling the convergence rate while balancing the communication load across nodes, the proposed Ramanujan-graph-based solution can achieve a better loss/accuracy at the same maximum energy consumption per node (`Ramanujan $50\%$'), or $37.5\%$ lower maximum energy consumption per node with a loss/accuracy no worse than `Vanilla D-PSGD' (`Ramanujan $31.3\%$'). 

Compared with the results in Fig.~\ref{fig:results_general}, the proposed solution delivers less energy saving at the busiest node under a fully-connected base topology. Intuitively, this phenomenon is because the symmetry of the base topology leads to naturally balanced loads across nodes even if this is not considered by the benchmarks, which indicates that there is more room for improvement in cases with asymmetric base topology. 

\fi

\end{document}